\documentclass[journal,comsoc]{IEEEtran}
\usepackage[T1]{fontenc}
\usepackage{cite}

\ifCLASSINFOpdf
   \usepackage[pdftex]{graphicx}
\else
   \usepackage[dvips]{graphicx}
\fi

\usepackage{amsmath}
\usepackage[cmintegrals]{newtxmath}

\usepackage[numbers]{natbib}
\usepackage{booktabs}
\usepackage{multirow}
\usepackage{multicol}
\usepackage{amsmath}
\usepackage{algpseudocode}
\usepackage{mathtools}
\usepackage[linesnumbered,ruled,vlined]{algorithm2e}

\newtheorem{theorem}{Theorem}
\newtheorem{lemma}{Lemma}
\newtheorem{proof}{Proof}
\newtheorem{proposition}{Proposition}
\newtheorem{corollary}{Corollary}

\mathchardef\mhyphen="2D

\begin{document}
%
\title{Active Thinking Model: A Goal-Directed Self-Improving Framework for Real-World Adaptive Intelligence}

\author{Hong~Su
\IEEEcompsocitemizethanks{\IEEEcompsocthanksitem H. Su is with the School of Computer Science, Chengdu University of Information Technology, Chengdu, China.\\
 E-mail: suguest@126.com. \\
\protect\\
}
\thanks{}}

\markboth{Journal of \LaTeX\ Class Files,~Vol.~14, No.~8, August~2015}%
{Shell \MakeLowercase{\textit{et al.}}: Bare Demo of IEEEtran.cls for IEEE Communications Society Journals}
%

\maketitle

\begin{abstract}
Real-world artificial intelligence (AI) systems are increasingly required to operate autonomously in dynamic, uncertain, and continuously changing environments. However, most existing AI models rely on predefined objectives, static training data, and externally supplied feedback, which restrict their ability to adapt, reflect, and improve independently. In this paper, we propose the \textit{Active Thinking Model} (ATM)—a unified cognitive framework that integrates goal reasoning, dynamic task generation, and self-reflective learning into an adaptive architecture. Unlike conventional systems that passively execute fixed procedures, ATM actively evaluates its performance through logical reasoning and environmental indicators, reuses effective methods to solve new problems, and generates novel strategies for unseen situations via a continuous self-improvement loop. A mathematically grounded theoretical analysis demonstrates that ATM can autonomously evolve from suboptimal to optimal behavior without external supervision and maintain bounded tracking regret under changing environmental conditions. 
\end{abstract}


\IEEEpeerreviewmaketitle






\section{Introduction}
The growing complexity and uncertainty of real-world environments increasingly demand intelligent systems that can reason, adapt, and improve autonomously without relying on continuous external supervision. Traditional artificial intelligence (AI) and machine learning models \cite{soni2025recent} — while highly effective in well - defined and static domains—largely depend on predefined objectives, fixed training data, and externally driven feedback. Such systems often struggle when faced with dynamic, unpredictable, or open-world settings, where environmental conditions, goals, and evaluation criteria evolve over time. 

Recent progress in large language models (LLMs) and autonomous agents has demonstrated preliminary forms of contextual reasoning and task adaptation \cite{naveed2025comprehensive} \cite{wang2025history}. However, these systems still operate under rigid task boundaries and lack mechanisms for self-evaluation, long-term improvement, and environment-aware decision reconfiguration. In particular, they do not actively generate new goals, refine their internal methods, or reconcile contradictions between intended objectives and real-world outcomes. This limitation highlights the need for an intelligent framework that integrates reasoning, reflection, and self-improvement into a unified architecture.

To address these challenges, this paper proposes the \textit{Active Thinking Model} (ATM)—a real-world–aware, self-reflective architecture that enables AI systems to think and act autonomously. Unlike conventional task-execution frameworks, ATM introduces an internal \textit{thinking loop} that continuously plans, evaluates, and improves its own behavior. The model is designed around three key principles: (1) \textbf{goal-conditioned reasoning}, which allows the system to dynamically adjust its behavior based on explicit and implicit objectives; (2) \textbf{scenario-separated memory}, which records contextual information about environments, goals, and outcomes to guide future decisions; and (3) \textbf{continuous self-improvement}, achieved autonomously through internal reflection, simulation-based verification, and adaptive task reconfiguration.

Formally, ATM integrates these principles into a hierarchical architecture consisting of environmental perception, goal reasoning, dynamic tasking, self-evaluation, and reflective learning. Each component interacts through mathematically defined processes that ensure feedback consistency, performance convergence, and long-term adaptability. The system operates in a closed cognitive loop—collecting environmental information, planning and executing goal-oriented tasks, evaluating both direct and indirect outcomes, and refining its internal models accordingly.

The main contributions of this work can be summarized as follows:
\begin{enumerate}
    \item We propose a unified \textbf{Active Thinking Model (ATM)} that integrates goal reasoning, dynamic task generation, and self-reflective learning into a single adaptive cognitive architecture capable of autonomous method reuse, continuous self-improvement, and effective handling of novel problems in dynamic real-world environments.
    
    \item We develop a mathematically grounded framework for \textbf{continuous self-improvement and active measurement}, demonstrating that ATM can autonomously evolve from suboptimal to optimal behavior while maintaining bounded tracking regret under environmental changes. The model continuously measures, evaluates, and compares alternative methods through internal feedback and state-difference analysis, enabling reliable performance improvement even under uncertain or partially observable conditions without relying on external supervision.

    \item We design a \textbf{scenario-separated memory mechanism} that connects environmental states, goals, and actions, enabling context-aware reasoning, efficient experience reuse, and lightweight memory reconstruction through predictive inference.
\end{enumerate}

The remainder of this paper is organized as follows. Section~\ref{sec_related_work} reviews related research on adaptive reasoning, self-reflective architectures, and large-model–driven autonomous systems. Section~\ref{sec_atm_design} introduces the proposed \textit{Active Thinking Model} (ATM) and its real-world–aware design principles. Section~\ref{sec_theory_analysis} provides the theoretical analysis demonstrating the model’s capability for autonomous self-improvement and adaptability. Finally, Section~\ref{sec_conclusion} concludes the paper and discusses potential extensions for future research.

\section{Related Works} \label{sec_related_work}

\subsection{LLM-Based Agents, Planning, and Reflection}
Large language model (LLM)–driven agents have recently become a central topic in AI research, motivating systems that perform multi-step reasoning, planning, and tool use through natural-language instructions. Early frameworks such as Chain-of-Thought prompting and ReAct demonstrated that explicit reasoning traces can improve interpretability and decision quality by decomposing problems into intermediate steps \cite{wei2022cot,yao2023react}. Subsequent works like Reflexion \cite{shinn2023reflexion} and Toolformer \cite{schick2023toolformer} extended this paradigm with feedback loops and external tool invocation, allowing agents to revise outputs based on self-generated critiques or retrieved evidence. More recent multi-agent and planner–executor designs explore collaborative reasoning and dynamic decomposition, revealing that large models can coordinate specialized sub-modules to execute complex goals.

Despite this progress, most existing LLM-based agents operate under static prompting and episodic contexts—each reasoning session resets once a task ends. Their self-correction typically remains bounded within a single episode and lacks mechanisms for persistent memory, goal evolution, or long-term performance tracking. In contrast, the proposed \textit{Active Thinking Model} (ATM) introduces a \emph{goal-conditioned}, \emph{scenario-aware} control loop that maintains stateful memory and supports autonomous task creation. Rather than relying on predefined prompt templates, ATM dynamically generates, evaluates, and reuses methods across tasks and environments. This shift transforms reflection from an episodic post-hoc process into a continuous, environment-embedded cycle of learning and adaptation, enabling long-horizon self-improvement beyond the scope of current LLM-agent frameworks.

\subsection{Goal-Conditioned Control, Hierarchical RL, and POMDPs}
Goal-conditioned control has been extensively studied in the reinforcement learning (RL) community, where agents learn to achieve predefined objectives through interactions with the environment. Traditional frameworks, such as Markov decision processes (MDPs) and partially observable MDPs (POMDPs), formalize sequential decision-making under uncertainty \cite{sutton2018rl,kaelbling1996pomdp}. Hierarchical RL further decomposes complex problems into subgoals and subtasks, allowing agents to reuse learned policies through temporal abstraction \cite{barto2003options}. These paradigms emphasize reward optimization through repeated trial and error, often assuming that goals and reward functions are explicitly specified and remain stable throughout training. Similarly, model-predictive control (MPC) and goal-conditioned policies \cite{mayne2000mpc} focus on mapping goals to low-level actions under dynamic constraints, achieving impressive results in robotics and continuous control.

However, these frameworks share several inherent limitations when applied to open, evolving real-world environments. First, the definition of goals and rewards is typically static or manually crafted, making it difficult for agents to autonomously redefine or reprioritize objectives when conditions change. Second, adaptation usually occurs through parameter updates over long training horizons, rather than through immediate structural or cognitive reconfiguration. Third, the feedback loop in RL-based systems is externally driven by reward signals, which may not always capture higher-level reasoning or ethical considerations. In contrast, the proposed \textit{Active Thinking Model} (ATM) extends beyond classical RL by integrating \emph{goal reasoning}, \emph{scenario–goal mappings}, and \emph{dynamic method evolution} within a unified cognitive framework. Instead of optimizing a fixed reward, ATM continuously evaluates its own goals, adjusts task hierarchies in real time, and replaces outdated methods with newly generated ones. This enables a higher level of autonomy and interpretability, bridging the gap between symbolic goal reasoning and statistical learning in dynamic environments.

\subsection{Continual Learning and Lifelong Adaptation}
Continual learning aims to enable models to accumulate knowledge across multiple tasks without catastrophic forgetting. Traditional approaches mitigate forgetting through regularization-based methods that constrain parameter drift \cite{kirkpatrick2017ewc}, replay-based methods that revisit stored samples or latent representations \cite{rebuffi2017icarl}, and architectural modularity that dynamically expands or allocates subnetworks for new tasks \cite{delange2021continual}. Lifelong and meta-learning extend these principles by emphasizing generalization and knowledge transfer across tasks \cite{hewitt2020meta}. These approaches have achieved success in classification, robotics, and language modeling, allowing agents to retain and reuse prior knowledge in related contexts. However, they generally assume that tasks are predefined and arrive sequentially, with limited capacity for autonomous task discovery or reorganization once the task set is fixed.

Despite their progress, most continual learning systems remain primarily reactive: they adapt only when new labeled data or explicit task boundaries are presented. They also depend on static evaluation metrics and externally defined loss functions, which prevent agents from recognizing when existing methods become outdated or suboptimal in dynamic environments. The proposed \textit{Active Thinking Model} (ATM) redefines continual learning as an \emph{active and self-driven} process rather than a passive response to new data. Through its utility-based task lifecycle—comprising creation, suspension, and replacement—ATM autonomously manages its internal objectives and continuously seeks to improve them. Its \emph{scenario-separated memory} further supports lifelong adaptation by conditioning method retrieval and goal selection directly on environmental context, enabling knowledge reuse across evolving domains. In this way, ATM establishes a persistent, context-aware adaptation mechanism that maintains learning continuity even in the absence of explicit task supervision.

\subsection{Memory Architectures and Knowledge Retrieval}
Memory-augmented neural architectures have long sought to extend a model’s ability to reason over long temporal horizons and complex dependencies. Classical designs such as Neural Turing Machines (NTM) and Differentiable Neural Computers (DNC) \cite{graves2016dnc} introduced external differentiable memory modules that allow neural networks to read and write information across time steps, improving their capacity for sequence reasoning and algorithmic tasks. More recent retrieval-based frameworks, such as Retrieval-Augmented Generation (RAG) \cite{lewis2020rag} and other retrieval-enhanced transformers, integrate external knowledge sources into large language models (LLMs), enabling them to access factual or context-specific information dynamically. In parallel, episodic memory systems \cite{borji2023memories} and hierarchical memory networks have been proposed to model cognitive recall and multi-level information abstraction, providing improved interpretability and explainability in reasoning processes.

While these architectures have advanced long-term reasoning, they often emphasize retrieval accuracy rather than adaptive organization of stored information. Most rely on static key–value mappings or attention-based similarity metrics, which do not evolve with environmental context or task dynamics. Consequently, conventional memory systems can recall relevant information but lack the capacity to restructure, filter, or reinterpret experiences as conditions change. The \textit{Active Thinking Model} (ATM) addresses this gap through a structured and hierarchical \textit{scenario-separated memory}. This memory explicitly links environmental states, goals, and actions, organizing information into layered mappings such as \textit{question→solution} and \textit{scenario→goal/action}. By coupling retrieval with contextual adaptation and predictive reconstruction, ATM transforms memory from a passive storage unit into an active reasoning substrate that guides method selection, reflection, and future goal formation. This design enables adaptive knowledge reuse and dynamic reflection beyond what current memory architectures achieve.

\subsection{Online Learning, Bandits, and No-Regret Selection}
Online learning and multi-armed bandit frameworks provide the mathematical foundation for decision-making under uncertainty, where an agent must sequentially select actions and minimize cumulative regret compared to an optimal strategy. Foundational algorithms such as Upper Confidence Bound (UCB) \cite{auer2002ucb} and stochastic approximation methods \cite{robbins1951stochastic} formalized the trade-off between exploration and exploitation, establishing provable performance bounds. Extensions through online convex optimization \cite{cesa2006prediction} and adversarial bandit models further generalized these ideas to dynamic and non-stationary settings. These methods have found wide application in adaptive control, personalized recommendation, and reinforcement learning, providing guarantees that the agent’s performance converges to the optimal solution over time.

However, most online learning and bandit-based methods assume explicit numerical rewards and well-defined cost functions. They optimize quantitative objectives rather than qualitative or context-dependent goals and typically require continuous external feedback for performance evaluation. In contrast, the \textit{Active Thinking Model} (ATM) generalizes the no-regret principle beyond scalar rewards to encompass \emph{goal-weighted qualitative evaluation}. ATM continuously measures and compares the performance of multiple candidate methods using direct outcomes, environmental indicators, and reflective reasoning. Its decision-making process is guided by adaptive thresholds that determine when to create, replace, or retire methods and tasks, allowing it to maintain efficiency under changing conditions. 

\subsection{Anomaly Detection, Change Detection, and Indirect Evaluation}
Anomaly and change detection have long been studied in signal processing, statistics, and machine learning as essential tools for identifying deviations from nominal behavior. Classical methods such as likelihood-ratio tests and cumulative sum (CUSUM) algorithms \cite{basseville1993detection} enable sequential detection of abrupt distributional shifts, while density-based and clustering-based techniques identify outliers in high-dimensional data \cite{chandola2009anomaly}. In recent years, deep learning–based detectors, including autoencoders and graph neural networks, have extended these principles to complex, non-linear domains such as cybersecurity, sensor networks, and autonomous vehicles. These approaches typically rely on reconstructing or predicting expected behavior and flagging deviations as anomalies, forming the basis of reliability monitoring and fault diagnosis in intelligent systems.

While these techniques effectively signal abnormal events, they often treat anomaly detection as an isolated diagnostic function rather than an integrated cognitive process. They lack mechanisms to connect deviations with goal fulfillment or to trigger adaptive reasoning in response to detected irregularities. The \textit{Active Thinking Model} (ATM) operationalizes these concepts within its \emph{active measurement} and \emph{reflection} modules. Deviations—termed \textit{flags}—are detected not only from primary outputs but also from indirect environmental or behavioral indicators. These may include inconsistencies with goals, contradictions in logical reasoning, or abnormal sensor patterns. Once detected, such deviations prompt ATM to initiate reflection, method replacement, or dynamic task adaptation, effectively closing the loop between perception, evaluation, and cognitive response. In this way, ATM fuses anomaly detection with self-assessment, transforming passive monitoring into proactive self-correction that sustains performance in uncertain environments.

\subsection{Simulation-Based Verification and Digital Twins}
Simulation-based verification and digital-twin technologies have emerged as critical tools for testing, validating, and optimizing complex systems under diverse operational conditions. In traditional engineering contexts, simulation has been used for design verification, safety assessment, and fault prediction, allowing virtual experimentation before physical deployment. The concept of the \textit{digital twin} extends this by maintaining a dynamic, real-time digital replica of a physical system that mirrors its states, behaviors, and interactions \cite{twin_survey,schleich2019digital}. These virtual environments enable counterfactual reasoning, sensitivity analysis, and predictive maintenance, providing a safe and cost-effective means to analyze system responses to hypothetical scenarios. As a result, digital-twin technology has found widespread application in manufacturing, aerospace, autonomous driving, and smart infrastructure, where reliability and explainability are paramount.

Despite their effectiveness, most existing simulation and digital-twin frameworks function as external validation layers—tools used to assess predefined models rather than integral components of autonomous cognition. They rarely participate in the reasoning or decision-making processes of AI systems in real time. The proposed \textit{Active Thinking Model} (ATM) embeds simulation-based comparison directly into its cognitive architecture, transforming simulation from a passive testing tool into an active reasoning mechanism. When discrepancies arise between predicted and observed outcomes, ATM initiates internal simulation to reconstruct potential causal paths and identify hidden errors or unobserved environmental factors. This process allows the model to validate and refine its own decisions without external supervision. By merging simulation-based verification with its reflective and goal-oriented reasoning, ATM achieves continual consistency checking and predictive foresight, ensuring safer and more reliable behavior in dynamic real-world environments.

\section{Active Thinking Model in Real-World Contexts} \label{sec_atm_design}

\subsection{Problem Setting and Challenges} \label{sec_problem_challenges}
Artificial intelligence (AI) systems are increasingly deployed in real-world scenarios such as autonomous driving and field robotics. Unlike controlled environments, the real world is inherently dynamic, uncertain, and continuously evolving. Fixed strategies or static models are therefore insufficient for robust adaptation. An AI system must not only execute pre-defined methods but also generate, evaluate, and refine new methods suited to its current environment.

To achieve effective interaction with the real world, an intelligent agent must satisfy two key requirements. First, it must detect the dominant features of the current environment, identify risks, and dynamically set appropriate goals based on situational context. Second, it must continuously improve its reasoning and operational methods to maintain performance across changing conditions. For example, in autonomous driving, when road conditions become hazardous, the system should autonomously prioritize safety and adjust its behavior accordingly.

These requirements imply that real-world AI systems must go beyond static policy learning and incorporate mechanisms for goal adaptation, self-improvement, and context-aware method generation. Such capabilities form the foundation of the proposed \textit{Active Thinking Model} (ATM), which supports autonomous reasoning, adaptive goal management, and dynamic method evolution.

\subsection{Real-World–Aware Active Thinking Model (ATM)} \label{sec_atm_design}

The proposed \textit{Active Thinking Model} (ATM) is designed to enable an AI system to think and act autonomously. In this context, “thinking” refers to the system’s capability to form goals, act toward them, and continuously improve itself—performing tasks more effectively and developing better methods over time. Rather than executing externally prescribed procedures, ATM actively formulates and carries out its own tasks to adapt to real-world complexity. 

ATM can select the most appropriate method from its learned repertoire based on internal evaluation. It judges which method performs better under the current environmental conditions and can emphasize specific goals when necessary. For example, in an autonomous driving scenario, if the environment becomes hazardous, the model can automatically prioritize safety and adjust its control strategies accordingly.

Both methods and goals are drawn from accumulated experience, as described in Section~\ref{sec_scenario_separate_memory}. While the underlying task plans are initially generated by a large language model (LLM), ATM continuously records, evaluates, and refines them to identify which methods are best suited for the current situation. This integration of autonomous goal setting, method selection, and continuous self-assessment allows ATM to maintain robust performance in dynamic real-world environments.

\subsubsection{Goal-/Target-Conditioned Tasks} \label{sec_task_with_goal}
A \textit{task} in the Active Thinking Model (ATM) is defined as a sequence of actions executed to achieve a specific goal. Formally, a task $\mathcal{T}$ can be represented as an ordered set of actions:
\[
\mathcal{T} = \{a_1, a_2, \ldots, a_k\}, \quad a_i \in \mathcal{A},
\]
where $\mathcal{A}$ denotes the space of all possible actions. Each action $a_i$ is influenced by environmental conditions, internal states, and accumulated experience stored in memory.

Before execution, ATM establishes one or more \textit{goals} $G = \{g_1, g_2, \ldots, g_m\}$ that guide reasoning and decision-making throughout the task. These goals act as high-level constraints rather than executable operations. At each step $t$, the large language model (LLM) receives the current goal $g_t$ as a global conditioning parameter and generates the next action according to:
\[
a_{t+1} = f_{\text{LLM}}(E_t, S_t, g_t),
\]
where $E_t$ represents the environmental state and $S_t$ the system state at time $t$. This formulation ensures that every generated action remains consistent with the active goal. For instance, in autonomous driving, if “safety” is defined as the goal, each planned maneuver must satisfy safety-related constraints during both planning and execution.

ATM continually monitors goal achievement using a goal-compliance function $C(g_t, E_t, S_t)$, which evaluates whether progress toward the goal has been made:
\[
C(g_t, E_t, S_t) =
\begin{cases}
1, & \text{if goal } g_t \text{ is satisfied},\\
0, & \text{otherwise}.
\end{cases}
\]
If $C(g_t, E_t, S_t) = 0$, ATM re-invokes the LLM to replan or adjust future substeps to restore goal alignment. This adaptive, goal-conditioned control contrasts with traditional rule-based systems, which rely on static task definitions.

Goals may be either \textit{explicit} (e.g., ensuring safety or accuracy) or \textit{implicit} (e.g., improving efficiency or energy saving). Both types are retained in memory for future reference and reuse. Importantly, goals are not static—they can evolve according to environmental changes or the internal state of the system. Some goals are selected because ATM has previously learned their importance, while others emerge dynamically from current situational demands. By maintaining a scenario–goal mapping, as described in Section~\ref{sec_scenario_separate_memory}, ATM can identify and select the most appropriate goal for each real-world context.

\subsubsection{Continuous Self-Improvement via Background Tasks} \label{sec_always_running}
The Active Thinking Model (ATM) should incorporate a built-in mechanism for continuous self-improvement. It autonomously refines its methods based on real-world interaction, learning from both failures and successes. Improvement is triggered not only when undesirable results occur but also when more effective methods are discovered, allowing ATM to progressively accumulate experience and expand its internal repertoire of strategies.

To achieve this, ATM maintains a set of background or auxiliary tasks dedicated to monitoring and enhancing its own performance. These self-improvement tasks are not predefined by the user but are instead generated by the system itself—either in response to observed environmental changes or to ensure the stability and correctness of ongoing primary tasks, as in section \ref{sec_spare_tasks}. Through continuous observation and information collection, ATM decides when to initiate, modify, or terminate these auxiliary tasks.

\subsubsection{Dynamic Tasking and Method Evolution} \label{sec_dynamic_thinking}

In the Active Thinking Model (ATM), all components—including tasks, goals, and improvement mechanisms—are dynamic. Each element can be created, modified, or replaced as the system acquires new experiences or encounters novel environmental conditions. This ensures that ATM continuously adapts its structure and reasoning strategies to meet changing real-world demands.

Let $\mathcal{T} = \{T_1, T_2, \ldots, T_n\}$ represent the set of active tasks at time $t$. Each task $T_i$ is defined by a tuple
\[
T_i = (G_i, A_i, E_t, S_t),
\]
where $G_i$ is the task goal, $A_i$ the action sequence, and $(E_t, S_t)$ denote the environmental and system states. ATM evaluates the relevance and performance of each task through a task-utility function:
\[
U(T_i) = f_u(G_i, E_t, S_t),
\]
and dynamically updates the task set as
\[
\mathcal{T}_{t+1} = \mathcal{T}_t - \{T_i \mid U(T_i) < \theta_d\} + \{T_j^{\text{new}} \mid U(T_j^{\text{new}}) > \theta_c\},
\]
where $\theta_d$ and $\theta_c$ are thresholds for task deletion and creation, respectively. This formulation enables ATM to autonomously discard low-utility tasks and instantiate new ones when improved or contextually relevant methods are detected.

Tasks within ATM are therefore fully replaceable and self-regulating. When the system identifies a performance degradation or detects opportunities for enhancement, it can autonomously generate and execute new processes—such as environmental sensing, memory reorganization, or internal reasoning—while suspending or revising ineffective ones. Even predefined maintenance routines can be modified or restored if the system determines that a previous change negatively impacts performance.

Furthermore, the mechanisms governing method evaluation and improvement are also adaptive. When a superior evaluation criterion $f_p'$ outperforms the current one $f_p$, ATM replaces or refines it according to
\[
f_p \leftarrow (1 - \beta)f_p + \beta f_p', \quad \beta \in (0,1),
\]
where $\beta$ controls the adaptation rate. This recursive adaptability ensures that both the reasoning process and the improvement strategy evolve jointly over time, allowing ATM to maintain long-term robustness, stability, and relevance in dynamic real-world environments.

\subsubsection{System Architecture of Active Thinking Model (ATM)} \label{sec_atm_architecture}

The overall architecture of the Active Thinking Model (ATM) is designed to enable autonomous perception, goal-oriented decision-making, and continuous self-improvement through a closed feedback loop. The perception component identifies issues and key environmental factors that must be resolved or considered before further actions are taken, while the decision-making component formulates strategies and goals accordingly. The feedback mechanism ensures that the system continually refines its methods based on real-world outcomes. Figure~\ref{fig_atm_architecture} conceptually illustrates the core components and their interactions.

ATM consists of four major functional modules:

\begin{enumerate}
    \item \textbf{Environmental Perception and Monitoring Module:}
    This module continuously collects information from both the external environment and the internal state of the system. It detects abnormalities, identifies deviations from expected patterns, and records contextual features for later analysis and reflection. Its primary function is to identify the key issues and environmental factors that must be resolved or addressed before further actions are taken, thereby providing the foundation for informed decision-making in subsequent modules.

    \item \textbf{Task Generation and Management Module:}
    Tasks are the basic operational units of ATM. This module is responsible for creating, scheduling, and managing three categories of tasks: 
    (i) \textit{self-improvement tasks} that refine internal methods; 
    (ii) \textit{adaptive tasks} that respond to environmental or contextual changes; and 
    (iii) \textit{goal-driven tasks} externally assigned or derived from user requests. 
    Abstract task descriptions are converted into executable sequences through the LLM planner.

    \item \textbf{Optimization and Evaluation Module:} 
    This module maintains a record of all executed methods, their associated contexts, and observed outcomes. It evaluates task performance, compares alternative methods, and identifies those that yield superior results. Methods that perform poorly can be refined, replaced, or removed, ensuring that the model evolves toward higher efficiency and stability.

    \item \textbf{LLM-Driven Planning and Reasoning Module:}
    The large language model serves as the cognitive core that transforms high-level goals into concrete task plans. During execution, it provides reasoning support, replans sub-steps when deviations occur, and integrates experiential knowledge stored in the scenario-based memory (Section~\ref{sec_scenario_separate_memory}).
\end{enumerate}

These modules operate in a continuous cycle. The perception module feeds real-time observations to the task generation module, which determines whether new actions or improvements are required. The optimization module evaluates completed tasks and updates stored experience, while the LLM planner refines reasoning strategies based on this evolving knowledge. 

To prevent the system from becoming static, ATM includes a dynamic updating mechanism that allows any task, method, or improvement process to be added, modified, or replaced. This ensures that the system remains adaptive and capable of lifelong learning in complex and changing environments.

\begin{figure}[!t]
\centering
\includegraphics[width=3.8in]{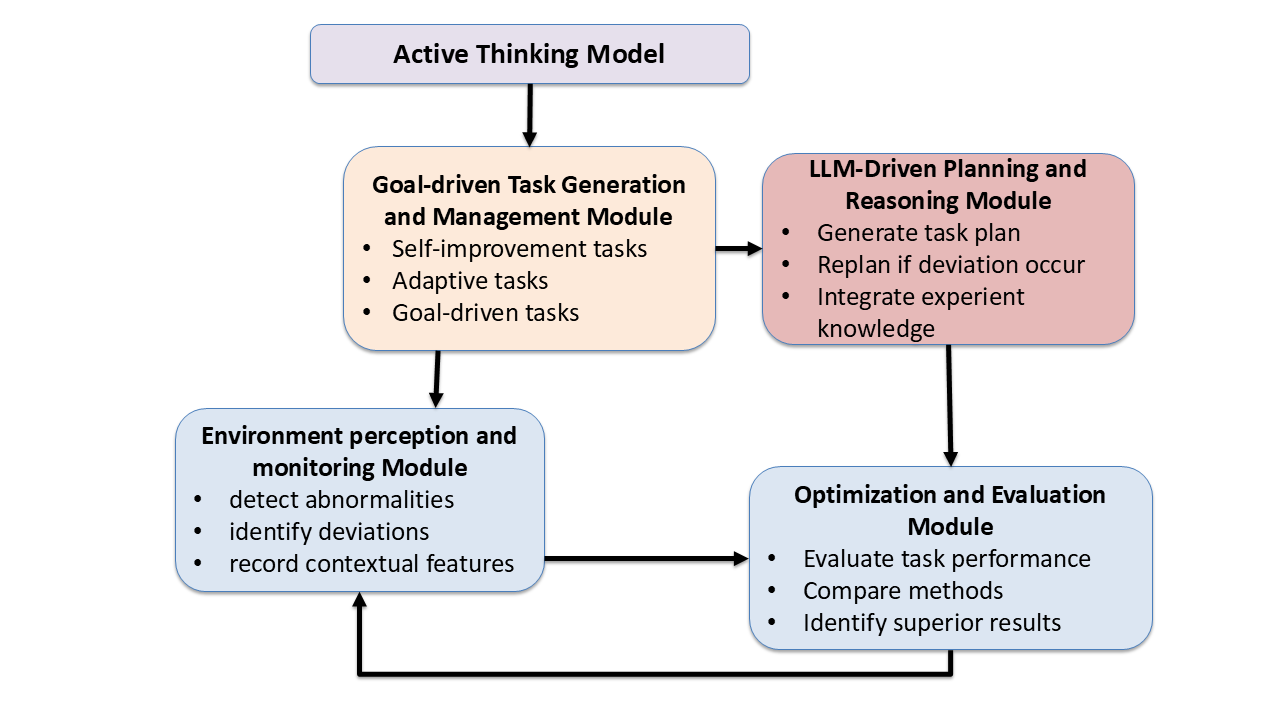}
\caption{System architecture of the proposed Active Thinking Model (ATM).}
\label{fig_atm_architecture}
\end{figure}

\subsection{Task Execution with Environmental Checkpoints and Method Selection} 
\label{sec_task_executed}
Before task execution, the large language model (LLM) generates an initial \textit{draft plan} based on currently available information. This plan specifies how the task should progress through a series of high-level stages without executing the actual actions. To ensure correctness during execution, the plan is divided into several \textit{environmental checkpoints}, each representing a key observable condition or milestone in the external environment.

At each checkpoint, ATM compares the current environmental state and task progress with the expected state defined in the plan. The degree of deviation between the actual and planned conditions, denoted by $\Delta E_t = \|E_t - \hat{E}_t\|$, is used to evaluate consistency. If $\Delta E_t$ exceeds a threshold $\theta_c$, the system triggers a corrective mechanism—either refining the plan or invoking the LLM to generate an adjusted subplan. These checkpoints thus serve both as validation anchors and as guidance points to steer subsequent actions toward goal compliance.

Once the Active Thinking Model (ATM) has planned a task with its associated goals—as described in Section~\ref{sec_task_with_goal}—using its most suitable learned methods, the system proceeds to execute it. During execution, ATM continuously ensures that each step remains aligned with both the predefined goals and the evolving environmental context. Whenever discrepancies or unfulfilled goals are detected, the system dynamically replans using the LLM, ensuring real-time adaptation and consistency between internal reasoning and external conditions.

\subsubsection{Environment Sensing: Main Issues and Change Detection} \label{sec_env_changes}
During real-world task execution, continuous environmental sensing is essential. The objective of this process is twofold: first, to identify the key issues and contextual factors that the system must address; and second, to detect any significant changes that could affect subsequent actions. Environmental information includes spatial factors, the behaviors of surrounding objects, and other situational parameters relevant to the current task.

In practice, the Active Thinking Model (ATM) provides the collected environmental information $E_t$ to the large language model (LLM), which analyzes it to determine the most appropriate next actions. 

As environmental conditions evolve, ATM continuously compares the current state $E_t$ with prior observations $E_{t-1}$ to identify significant temporal or spatial differences:
\[
\Delta E_t = \| E_t - E_{t-1} \|,
\]
where $\Delta E_t$ quantifies the degree of environmental change. If $\Delta E_t$ exceeds a predefined threshold $\theta_e$, the system triggers task adaptation or re-planning to maintain performance stability \cite{su2025difference}. This continuous monitoring enables ATM to respond proactively to dynamic real-world variations.

Information collection must be timely and persistent. Real-time or near-real-time sensing enables ATM to respond quickly to emerging risks or anomalies, while sustained monitoring prevents the system from overlooking gradual but critical environmental shifts. The ultimate goal of this process is to maintain an up-to-date understanding of environmental status and dynamic changes, providing a reliable basis for adaptive action planning and decision-making.

\subsubsection{Method Selection Under Uncertainty and Unknowns} \label{sec_method_choice_unknowns}

When selecting an appropriate method for a specific scenario, the Active Thinking Model (ATM) queries its knowledge base using both the current question and the contextual scenario to identify the most suitable learned method. In previous work, a direct mapping between questions and solutions was established \cite{su2025method}; in real-world interaction, ATM extends this mapping by incorporating scenario information to account for environmental variability.

For unknown or previously unencountered situations, the Active Thinking Model (ATM) retrieves the question to be solved, denoted as $Q_u$, and queries the large language model (LLM) to identify the most similar known method $m^{*}$ capable of addressing analogous cases. Formally, the system estimates similarity as
\[
m^{*} = \arg\max_{m_i \in \mathcal{M}} \, \text{Sim}(Q_u, Q_i),
\]
where $\text{Sim}(\cdot)$ measures the semantic or contextual similarity between the unknown question $Q_u$ and previously stored questions $Q_i$. This process follows the principle of \textit{cross-question method reuse} proposed in \cite{su2025cross}, enabling ATM to generalize learned methods and apply them effectively to novel contexts.

In highly urgent situations requiring immediate response, ATM can bypass deliberative reasoning and employ an \textit{intuition-based} strategy \cite{su2025layered}, directly applying the most probable or instinctive action generated by the LLM. This multi-level selection mechanism—ranging from scenario-based reasoning to intuition-driven reaction—enables ATM to maintain both responsiveness and adaptability under uncertainty.

\subsection{Scenario-Separated Memory with Prediction and Environment Context} \label{sec_scenario_separate_memory}

The memory mechanism in the Active Thinking Model (ATM) is designed to facilitate self-improvement by identifying the potential causes of incorrect method selection and by analyzing why certain methods yield superior outcomes. Through continuous accumulation of experiential data, the system incrementally refines both its decision-making processes and its goal-setting strategies.

To determine more suitable methods or assign more effective goals for a given task, ATM maintains structured mappings among problems, solutions, and environmental contexts. As introduced in the method-based approach \cite{su2025method}, questions are associated with their corresponding solutions. In this work, we extend that framework by incorporating an additional layer that maps each \textit{scenario}—representing the current environmental state and internal system condition—to an appropriate goal. This dual mapping of \textit{question–solution} and \textit{scenario–goal} relationships enables ATM to reason contextually and select optimized strategies for dynamic real-world interaction.

\subsubsection{Goal Storage and Scenario-Keyed Indexing} \label{sec_goal_storage}

Building on the method-based approach \cite{su2025method}, where each question is mapped to its solution, the Active Thinking Model introduces goal–scenario associations to improve contextual adaptability. A \textit{scenario} is defined as the state of both the external environment and the internal condition of the system, excluding explicit actions or goals.

Accordingly, ATM stores goal information as mappings from scenarios to corresponding goals.

Let $\mathcal{S}$ be the space of scenarios, $\mathcal{G}$ the set of goals, $\mathcal{A}$ the action set, and $\mathcal{O}$ the outcome space. 
ATM maintains two scenario-indexed memories:

\paragraph*{Scenario $\rightarrow$ Goal (weighted multi-map)}
\[
\mathcal{M}_{G}:\ \mathcal{S}\to \Delta(\mathcal{G}), 
\qquad 
\mathcal{M}_{G}(s)=\{(g,\;w_g(s)):\ g\in\mathcal{G}\},
\]
where $w_g(s)\ge 0$ is a stored preference (optionally normalized so $\sum_{g} w_g(s)=1$). 
Goal retrieval uses a maximum-weight rule
\[
g_t=\arg\max_{g\in\mathcal{G}} w_g(s_t),
\]
or a probabilistic draw $g_t \sim \text{Cat}(w_g(s_t))$ (where $\text{Cat}$ denotes a categorical distribution that samples goals according to normalized weights $w_g(s_t)$) if stochastic selection is desired.

\paragraph*{Scenario $\rightarrow$ Actions/Outcomes (structured memory)}
\[
\begin{aligned}
\mathcal{M}_{AO}:\ &\mathcal{S}\to \Delta(\mathcal{A}\times\mathcal{O}), \\[2pt]
&\mathcal{M}_{AO}(s)=\{((a,o),\,w_{a,o}(s)):\ a\in\mathcal{A},\,o\in\mathcal{O}\}.
\end{aligned}
\]

with $w_{a,o}(s)\ge 0$ (optionally $\sum_{a,o} w_{a,o}(s)=1$). 
Retrieval can target actions, outcomes, or joint pairs, e.g.,
\[
(a_t,o_t)=\arg\max_{(a,o)} w_{a,o}(s_t).
\]

\noindent\textit{Optional update:} after executing $(a_t,o_t)$ under $s_t$ and observing a goal-compliance signal $r_t$,
\[
w_{x}(s_t)\ \leftarrow\ (1-\eta_t)\,w_{x}(s_t)+\eta_t\,\phi(r_t),\quad x\in\{g\}\ \text{or}\ x\in\{(a,o)\},
\]
where $\eta_t>0$ is a step size and $\phi(\cdot)$ maps feedback to a weight increment.

This structure forms the foundation for selecting improved methods in future interactions, allowing the system to compare alternative approaches applied under identical or similar conditions.

Unlike traditional frameworks that only store question–solution pairs, the scenario-based extension establishes a hierarchical, multi-level memory. The first level maps questions to potential solutions, while the second refines these mappings based on specific scenarios, enabling context-aware method selection. Additional attributes—such as time, environmental category, or internal state—can be incorporated as extra layers, with each representing a distinct attribute dimension. This hierarchical design enables ATM to perform fine-grained retrieval and dynamic adaptation of methods in complex, changing environments.

\subsubsection{Process History with Prediction and Environmental Fill-In} \label{sec_process_history}

To facilitate effective reflection and improvement, the Active Thinking Model (ATM) maintains a temporal record of scenarios and their corresponding events. Each recorded entry includes a summarized description of the situation along with a timestamp whose duration depends on the specific scenario. This process history allows the system to revisit prior events—such as failure cases—to analyze potential causes and identify influential factors affecting task outcomes. 

The memory is organized chronologically, with more recent entries stored at the front. The recording interval can be adjusted according to task requirements, ensuring that key information is preserved while minimizing unnecessary data.

ATM also employs a predictive and environment-assisted mechanism to maintain a lightweight yet informative memory. Instead of storing every low-level detail, the system records only key actions or outcomes and relies on the large language model (LLM) to infer intermediate steps through predictive reasoning. Formally, given a sequence of recorded key states 
\[
\{s_1, s_k, s_T\}, 
\]
the LLM estimates the missing intermediate states by prediction:
\[
\hat{s}_t = f_{\text{pred}}(s_1, s_T, E_t), \quad t \in (1, T),
\]
where $f_{\text{pred}}(\cdot)$ denotes the LLM-based prediction function and $E_t$ represents environmental context.  

For common or repetitive real-world events, ATM stores only the deviation $\Delta E_t = E_t - E_{\text{norm}}$ from the environmental norm $E_{\text{norm}}$, reconstructing the complete context when necessary using surrounding environmental cues. This approach minimizes memory redundancy while preserving essential contextual information, ensuring that the system retains the knowledge needed for subsequent analysis, reflection, and improvement.

\subsection{Spare-Time Optimization for Method Improvement} \label{sec_spare_tasks}
The Active Thinking Model (ATM) periodically runs background or “spare-time” tasks dedicated to recalling past experiences, evaluating prior actions, and refining its internal methods. The objective of these tasks is to identify superior strategies and avoid previously ineffective ones, thereby improving future interactions with the real world. These spare-time tasks typically process recently executed operations that have not yet been fully analyzed.

The following subsections describe several key mechanisms employed for method improvement.

\subsubsection{Learning from Outcomes: Good vs.\ Bad Results} \label{sec_good_bad_results}
When an action produces favorable results, the AI system should store it for future reuse. Following the method-based framework \cite{su2025method} and section \ref{sec_goal_storage}, each method can be represented in a layered storage structure. The first layer maps a \textit{question} to a set of corresponding \textit{scenarios}, while the second layer maps each \textit{scenario} to its associated \textit{solution}. The scenario field captures the environmental conditions and task context in which the method was applied, since a successful method in one scenario may not be suitable for another. This hierarchical representation enables ATM to retrieve and apply the most context-appropriate method when similar situations arise.

For actions leading to undesirable results, the system recalls the corresponding scenario and analyzes potential causes, including environmental factors, the chosen actions, and the system’s internal state. If a method that previously yielded good results performs poorly under a new scenario, the differing environmental and task conditions are recorded to refine future selection. When faced with similar conditions, ATM queries the large language model (LLM) to determine whether the current scenario matches one that previously caused failure; if so, the associated method will be avoided.

If an action produces a poor outcome and no corresponding corrective method exists, two approaches can be applied. The first is to search for a similar scenario and retrieve an analogous method, as described in \cite{su2025cross}. The second approach is for ATM to employ adaptive tactics to mitigate the negative effects. Such tactics include gradually reducing the intensity of the current action, temporarily pausing execution, or reverting to a previously safe state. These responses are selected based on the severity of deviation. For example, if two autonomous vehicles approach an intersection too quickly and risk collision, both systems can reduce speed or stop to prevent failure. This tiered adjustment mechanism enables ATM to maintain operational stability and safety even under unexpected or high-risk conditions.

\subsubsection{Mining Always-Occurring Patterns (Temporal/Spatial)} \label{sec_always_patterns}
Certain events or phenomena tend to occur together, either sequentially over time or concurrently within the same spatial context. When one event consistently precedes another, it implies a \textit{temporal association}; conversely, when multiple events co-occur within the same environment or spatial region, it represents a \textit{spatial association}. Formally, let $\mathcal{E} = \{e_1, e_2, \ldots, e_n\}$ denote a set of observed events. A temporal association exists if
\[
P(e_j | e_i, t_j > t_i) > \theta_t,
\]
where $\theta_t$ is a threshold probability indicating a significant time-ordered dependency. Similarly, a spatial association is established when
\[
P(e_j | e_i, \text{loc}(e_j) = \text{loc}(e_i)) > \theta_s,
\]
with $\theta_s$ representing the spatial co-occurrence threshold. These relationships allow ATM to identify recurring event patterns that reveal causal or contextual dependencies across time and space.

In this paper, we primarily focus on temporal associations. During spare-time analysis, ATM recalls memories in chronological order to identify correlations, evaluate patterns, and refine its understanding of event dependencies. Even for complex spatial events, portions of the experience may be recorded at different times, which can still be treated as time-sequenced data.

The purpose of this process is to detect recurring event sequences and identify coherent patterns that may not be immediately obvious. Let events be represented as a time-ordered sequence 
\[
E = \{e_1, e_2, \ldots, e_T\},
\]
where \(e_t\) denotes an observed environmental or internal factor at time \(t\). ATM estimates the joint co-occurrence probability
\[
P(e_i, e_j) = \frac{\text{count}(e_i \rightarrow e_j)}{\text{count}(e_i)},
\]
and records a \emph{common occurrence sequence} if \(P(e_i, e_j) > \theta_c\), where \(\theta_c\) is a predefined coherence threshold. 

When two or more factors consistently appear together—either sequentially or concurrently—the system encodes their joint dependency as a coherence relation
\[
\mathcal{C} = \{(e_i,e_j)\,|\, P(e_i,e_j)>\theta_c\}.
\]
In subsequent reasoning, ATM ensures that related events in \(\mathcal{C}\) are considered jointly rather than independently. This formal mechanism enables the discovery of latent dependencies and supports the emergence of higher-level behavioral rules that generalize across contexts.

For instance, if a vehicle consistently experiences uneven motion on a particular road, ATM can infer a correlation between that road condition and the resulting vibration pattern. 

Additionally, ATM can decompose large processes into smaller, incremental steps, allowing each to be analyzed and improved individually. This granular refinement supports progressive optimization of system behavior over time.

\subsubsection{Reflective Analysis: Reasoning and Seeking Better Methods} \label{sec_reflection}
While the previous mechanisms focus on analyzing task outcomes, reflection emphasizes examining the internal processes of task execution. Reflection enables the Active Thinking Model (ATM) to identify influential factors at each stage of a task and determine how existing methods can be improved. Formally, a task $\mathcal{T}$ can be represented as a sequence of states and actions 
\[
\mathcal{T} = \{(s_1, a_1), (s_2, a_2), \ldots, (s_T, a_T)\}.
\]
During reflection, ATM provides the large language model (LLM) with a summarized trace of this sequence, denoted by $\Sigma(\mathcal{T})$, which encodes key transitions and contextual information. The LLM performs structured reasoning on this trace to produce improvement suggestions:
\[
R_{\text{LLM}} = f_{\text{reflect}}(\Sigma(\mathcal{T}), E_t, S_t),
\]
where $R_{\text{LLM}}$ represents the reasoning output and $f_{\text{reflect}}(\cdot)$ denotes the reflection function conditioned on environmental and internal states.

The results of this reflective analysis, $R_{\text{LLM}}$, are stored in memory as experience tuples for future adaptation and learning. This process allows ATM not only to refine its execution strategies but also to evolve its reasoning capability by iteratively improving the reflection mechanism itself.

Reflection in ATM is dynamic—new reflection strategies can be added, and existing ones can be modified or replaced as the system evolves. Two primary reflection approaches are employed:

\textbf{(1) Reasoning:}  
When encountering unresolved or novel issues, ATM conducts reasoning to uncover the underlying causes. This includes identifying the true nature of the problem, understanding the environmental context, and analyzing temporal and spatial relationships among events. Reasoning helps prevent repetitive mistakes by explaining why particular actions were taken and why certain external objects behaved as they did. 

Environmental indicators or “flags” are used as objective references; if the reasoning outcomes align with these external signals, the reasoning is considered valid. 

Each process stage is associated with its own reasoning record, which is cross-checked for consistency. If conflicts arise, ATM generates new reasoning to reconcile them, thereby improving the internal knowledge base.

\textbf{(2) Step-Level Improvement and Corrective Adjustment Mechanism:}  
ATM also evaluates whether existing methods can be enhanced. This process aims to discover improvements based on the currently deployed methods. Rather than discarding an entire method, ATM focuses on identifying one or several inefficient or overly complex steps within it that hinder performance or stability. For each method \(m_i\) represented as a sequence of actions \(\pi_i = \{a_1, a_2, \ldots, a_T\}\), the system computes a step-wise cost profile
\[
C(a_t) = \lambda_1\,T(a_t) + \lambda_2\,E(a_t) + \lambda_3\,R(a_t),
\]
where \(T(a_t)\), \(E(a_t)\), and \(R(a_t)\) denote time consumption, energy/resource usage, and associated risk respectively, and \(\lambda_k\) are weight coefficients. Steps with \(C(a_t) > \theta_{\text{eff}}\) (an efficiency threshold) are flagged as candidates for improvement.

ATM then explores corrective strategies at the step level, including:  
(i) \textit{degree reduction}—reducing the action’s intensity or duration (e.g., slower movement or partial computation);  
(ii) \textit{temporary suspension}—pausing the execution of a suspected problematic step to assess its necessity; and  
(iii) \textit{reversion}—rolling back the system state to a previously validated safe checkpoint if performance degradation is detected.  
Formally, the local improvement process minimizes expected cumulative cost subject to safety and goal constraints:
\[
a_t^{*} = \arg\min_{a_t' \in \mathcal{A}} \; \mathbb{E}[C(a_t')] \quad \text{s.t.} \quad R(\pi_i') \ge R(\pi_i), \; S(\pi_i') \ge S_{\min},
\]
where \(S(\pi_i')\) represents a safety score and \(S_{\min}\) the acceptable threshold.  
Through this adaptive correction, ATM ensures that improvements are made incrementally and safely—allowing continuous refinement of existing methods without external intervention.

Even when task outcomes are satisfactory, ATM may still analyze intermediate steps to discover opportunities for optimization, leading to gradual yet continual improvement. This proactive reflection ensures that both successful and failed experiences contribute to long-term performance refinement.

\subsection{Active Measurement and Indirect Evaluation} \label{sec_active_measurement}
The Active Thinking Model (ATM) does not rely solely on external feedback from users; instead, it actively collects information from the environment, its internal state, and task outcomes to perform self-assessment. By querying the large language model (LLM), ATM evaluates task performance both through direct results and through indirect indicators that reflect system behavior. This enables a more comprehensive understanding of task success, even when explicit evaluation metrics are unavailable.

When task outcomes can be measured directly—such as through quantitative performance metrics or observable deviations from predefined goals—ATM can readily determine correctness. However, in many real-world scenarios, outcomes are ambiguous or only partially observable. In such cases, ATM employs \textbf{indirect evaluation} mechanisms to infer performance quality. These mechanisms include logical analysis and feedback from the environment or other agents, encompassing the detection of deviations from normal patterns, contradictions, and anomalies in related environmental signals. The following subsections describe several of these indirect evaluation methods in detail.

\subsubsection{Deviation from Expected Outcomes via Real-World Flags} \label{sec_flags_deviation}
When results are not immediately interpretable, the Active Thinking Model (ATM) evaluates them against expected real-world indicators, referred to as “flags.” Each flag $F_i$ represents a measurable environmental or behavioral variable with an expected normal range $[F_i^{\min}, F_i^{\max}]$. Given an observed value $\hat{F}_i$, the system computes its deviation
\[
\Delta F_i = |\hat{F}_i - F_i^{\text{exp}}|,
\]
where $F_i^{\text{exp}}$ is the expected value under normal conditions. If $\Delta F_i > \theta_f$ for any $i$, where $\theta_f$ is the deviation threshold, ATM infers a potential error or inconsistency in the task outcome. This quantitative evaluation allows ATM to detect subtle anomalies even when the overall result is not directly interpretable.

Each task inherently carries an implicit goal representing an ideal or complete outcome. For instance, when the goal is to clean a soymilk machine, cleaning only the outer surface while neglecting internal components yields an incomplete result. In such comparisons, the method that produces a more comprehensive or higher-quality outcome is considered superior. To ensure fairness and completeness in evaluation, ATM extends the spatial and temporal scopes of analysis, capturing all relevant aspects of task performance before comparing results. This approach enables ATM to assess both the thoroughness and contextual appropriateness of task execution.

\subsubsection{Comparison of Uncertain States}
The execution of a method may alter not only the internal state of the Active Thinking Model (ATM) but also the condition of the surrounding environment and other interacting entities. These conditions—collectively referred to as \textit{uncertain states}—are not limited to fixed variables such as temperature or position. Instead, they represent dynamically evolving properties that vary across time and context.  

To assess the impact of an executed task, ATM compares the observed states before and after execution. Formally, the state difference can be represented as:
\[
\Delta S_t = S_{t+1} - S_t,
\]
where \(S_t\) denotes the composite environmental and internal state vector at time \(t\). A significant deviation \(|\Delta S_t| > \theta_s\) indicates that the task has materially influenced the system or environment.  

By quantifying these differences, the system can infer whether the task produced desirable or adverse effects. Since uncertain states may exhibit complex, nonlinear dependencies, ATM leverages the large language model (LLM) to identify and interpret such state variations using reasoning-based difference detection methods, similar to those proposed in \cite{su2025difference}.

\subsubsection{Contradiction Checks} \label{sec_contradictions}
Contradiction analysis provides an efficient mechanism for detecting incorrect or inconsistent task execution. Let the intended objective of a task be represented as $O_{\text{goal}}$ and the observed outcome as $O_{\text{obs}}$. A contradiction is identified when the similarity or directional alignment between the two falls below a critical threshold:
\[
\text{Sim}(O_{\text{goal}}, O_{\text{obs}}) < \theta_c,
\]
where $\text{Sim}(\cdot)$ measures semantic or behavioral consistency, and $\theta_c$ is the contradiction threshold. For instance, if an action intended to move east results instead in motion toward the west, the system immediately classifies the result as erroneous.  

To systematize this process, ATM prompts the large language model (LLM) to generate a set of verification dimensions, $\mathcal{D} = \{d_1, d_2, \ldots, d_k\}$, referred to as “check aspects.” Each dimension $d_i$ defines a measurable property or constraint (e.g., direction, magnitude, or causal effect). The actual execution trace is then compared against these generated dimensions. The detection of any logical or behavioral contradiction across $\mathcal{D}$ indicates that the current method or decision is inconsistent with the intended goal and therefore requires reevaluation or replanning.

\subsubsection{Indirect Signals: Errors, Failures, and Abnormalities} \label{sec_indirect_signals}
In some cases, a task’s direct output may appear acceptable, while indirect indicators reveal underlying failures or instabilities. These indirect factors—derived from environmental or behavioral cues—enable the Active Thinking Model (ATM) to detect hidden errors or abnormal system states that are not evident from the primary task result.  

Let $\mathcal{I} = \{I_1, I_2, \ldots, I_n\}$ denote the set of monitored indirect indicators, such as vibration intensity, sound frequency, or proximity variance. Each indicator $I_i$ has an expected normal value $I_i^{\text{exp}}$ under standard operating conditions. The system computes the aggregated deviation score:
\[
D_{\text{ind}} = \sum_{i=1}^{n} w_i \, |I_i - I_i^{\text{exp}}|,
\]
where $w_i$ represents the relative importance of each indicator. If $D_{\text{ind}} > \theta_{\text{ind}}$, where $\theta_{\text{ind}}$ is a learned abnormality threshold, ATM classifies the state as abnormal and initiates corrective action or reflection.  

For example, an autonomous vehicle may achieve its immediate goal of maintaining high speed yet still collide with another vehicle—a clear failure detectable through contextual cues such as collision force or unexpected deceleration patterns rather than the task outcome itself. This indirect-signal analysis allows ATM to recognize hidden risks and preserve overall system reliability.

ATM monitors such anomalies by analyzing environmental and internal signals, including unexpected sounds, irregular motions, or deviations from expected sensor readings. The LLM assists by suggesting which parameters or “measures” should be observed during execution, applying the difference-detection framework proposed in \cite{su2025difference}. If the deviation between expected and observed measures exceeds a defined threshold, the behavior is classified as abnormal.

These indirect signals are not direct components of task performance but rather auxiliary indicators reflecting system health and environmental consistency. By integrating these factors into its evaluation process, ATM can detect hidden or emerging problems early, enhancing robustness and reliability in real-world operation.

\subsubsection{External Feedback from Other Agents} \label{sec_external_feedback}
In addition to self-assessment and environmental monitoring, the Active Thinking Model (ATM) incorporates feedback from external agents—such as users, observers, or nearby entities—to refine its evaluation. Such feedback provides social or contextual cues that complement objective task outcomes. For example, if an action significantly alters the environment or elicits strong negative reactions from surrounding individuals (e.g., confusion, discomfort, or anger), it may suggest that the action was inappropriate or overly aggressive.  

Formally, let $\mathcal{F}_{\text{ext}} = \{f_1, f_2, \ldots, f_m\}$ denote the set of feedback signals from external agents, each normalized to the range $[0, 1]$, where higher values indicate greater approval. ATM computes an aggregate external feedback score:
\[
S_{\text{ext}} = \frac{1}{m} \sum_{j=1}^{m} f_j.
\]
If $S_{\text{ext}} < \theta_{\text{ext}}$, where $\theta_{\text{ext}}$ represents the minimum acceptable satisfaction level, the system infers that the action may have violated social norms or contextual appropriateness. The LLM is then prompted to interpret the collected data and generate a qualitative assessment of alignment between the observed outcomes and acceptable behavioral standards. This external feedback mechanism enables ATM to integrate social awareness into its adaptive evaluation process.

When evaluating between two functionally similar methods, ATM may prefer the one that causes less disruption or receives more favorable feedback. This sensitivity to external perception allows the system to align its actions not only with internal goals but also with external social and ethical considerations.

\subsubsection{Simulation-Based Comparison} \label{sec_simulation_comparison}
To detect unexpected or ambiguous outcomes, the Active Thinking Model (ATM) employs simulation-based analysis to verify correctness. In this process, the system replays or simulates the sequence of actions within the same environmental context to identify discrepancies between predicted and actual outcomes.  

Formally, let $\hat{O}_t = f_{\text{sim}}(E_t, S_t, \mathcal{A}_t)$ denote the simulated outcome produced by the simulation function $f_{\text{sim}}(\cdot)$, given the current environment $E_t$, system state $S_t$, and executed action sequence $\mathcal{A}_t = \{a_1, a_2, \ldots, a_T\}$. The observed real-world outcome is denoted by $O_t$. The system computes the simulation deviation as
\[
\Delta O_t = \| O_t - \hat{O}_t \|.
\]
If $\Delta O_t > \theta_{\text{sim}}$, where $\theta_{\text{sim}}$ is a threshold representing acceptable deviation, ATM infers that an inconsistency or failure has occurred. This comparison enables the system to detect hidden execution errors or incomplete processes that may not be immediately apparent through direct observation alone.

For instance, if the system initiates a phone call, it can simulate the expected behavior—such as detecting a dial tone—and compare it with the real-world response. If the dial tone is absent or the action terminates prematurely, ATM can infer that an anomaly has occurred. This simulation-based comparison enables the system to uncover hidden failures or incomplete executions that may not be immediately apparent through direct observation, thereby enhancing its ability to proactively detect and correct abnormal conditions.

\section{Theoretical Analysis of Self-Improvement and Adaptability} \label{sec_theory_analysis}
This section provides theoretical justification that the proposed Active Thinking Model (ATM) can (1) autonomously improve from poor to effective behaviors without external supervision, and (2) maintain adaptability in dynamic real-world environments. 

\subsection{Assumptions and Notation}
Let $m_t$ denote the method selected by ATM at time step $t$, and let $P_t(m_t)$ represent its estimated performance score. The environment and system states are denoted as $E_t$ and $S_t$, respectively. ATM receives an internal reward signal $r_t \in [0,1]$ that reflects the quality of its behavior, as computed from its indirect evaluation mechanisms (e.g., environmental flags, contradictions, or simulation results).  

ATM updates its internal evaluation function $f_p$ using a stochastic gradient-based rule:
\[
f_p \leftarrow f_p + \eta_t \nabla_{f_p} \mathcal{L}(r_t, P_t(m_t)).
\]
Here, $\eta_t > 0$ is the learning rate at time $t$, and $\nabla_{f_p} \mathcal{L}$ denotes the gradient of the loss function $\mathcal{L}$ with respect to $f_p$.  
The sequence $\{\eta_t\}$ satisfies $\sum_t \eta_t = \infty$ and $\sum_t \eta_t^2 < \infty$, which ensures convergence of stochastic approximation updates.  

We make the following assumptions:

\begin{itemize}
    \item \textbf{A1 (Internal reward signal):} The internal reward $r_t$ provides a proper monotonic estimate of task quality; that is, $\mathbb{E}[r_t \mid \text{true quality}]$ increases with actual performance.
    \item \textbf{A2 (Exploration):} With probability $\epsilon_t > 0$ and $\sum_t \epsilon_t = \infty$, ATM periodically explores alternative methods through self-initiated background tasks.
    \item \textbf{A3 (Stationarity and detectability):} The environment is piecewise stationary with change points $\{\tau_k\}$, which are detected when environmental deviation $\Delta E_t > \theta_e$ or outcome deviation $\Delta O_t > \theta_{\text{sim}}$ exceeds predefined thresholds.
    \item \textbf{A4 (Finite candidate set per regime):} In each stationary regime, there exists an optimal method $m^\star$ that maximizes the expected internal reward $\mathbb{E}[r_t]$.
\end{itemize}

\subsection{Improvement from Poor to Effective Behavior}

\begin{proposition}[Autonomous Method Improvement]
\label{prop:improvement}
Under A1–A4, within a stationary environment, the greedy choice
$m_t^{*} = \arg\max_i P_t(m_i)$
achieves sublinear regret compared to the optimal method $m^\star$:
\[
\mathbb{E}\!\left[\sum_{t=1}^{T} (r^\star - r_t)\right] = o(T),
\]
where $r^\star = \mathbb{E}[r_t \mid m^\star]$ denotes the expected reward of the best method. Consequently, the time-averaged performance satisfies
\[
\frac{1}{T}\sum_{t=1}^{T} r_t \to \mathbb{E}[r^\star] \quad \text{as } T \to \infty.
\]
\end{proposition}

\begin{proof}
Since $r_t$ is an unbiased internal feedback (A1) and exploration ensures non-zero probability of sampling all methods (A2), the stochastic update of $f_p$ (with $\eta_t$ satisfying the Robbins–Monro conditions) asymptotically estimates each method’s expected reward. The greedy policy then converges in mean to the optimal method $m^\star$, yielding sublinear regret as in standard online-learning analyses.
\end{proof}

\begin{proposition}[Monotone Goal-Compliance Improvement]
\label{prop:goal_compliance}
Let $C(g_t,E_t,S_t) \in \{0,1\}$ denote the goal-compliance indicator defined in Section~\ref{sec_task_with_goal}. If ATM triggers replanning whenever $C=0$, then
\[
\mathbb{E}[C_{t+1} \mid C_t = 0] \ge \mathbb{E}[C_t].
\]
\end{proposition}

\begin{proof}
Whenever non-compliance occurs ($C_t = 0$), ATM refines its method estimator $f_p$ using the observed feedback $r_t$. Since $\mathbb{E}[r_t]$ is positively correlated with true goal achievement, the probability of future compliance weakly increases in expectation, ensuring monotonic progress toward goal satisfaction.
\end{proof}

\begin{corollary}[Self-Improvement without External Supervision]
By Propositions~\ref{prop:improvement} and~\ref{prop:goal_compliance}, ATM’s internal reward, exploration, and update mechanism ensure that both internal reward $r_t$ and goal compliance $C_t$ improve over time. Thus, the model autonomously transitions from poor to effective behavior without any external intervention.
\end{corollary}

\subsection{Adaptability in Dynamic Environments}

\begin{theorem}[Bounded Tracking Regret under Environmental Change]
\label{thm:tracking}
Under A1–A4, with change detection and exploration reset after each change point $\tau_k$, ATM achieves tracking regret
\[
\mathbb{E}\!\left[\sum_{t=1}^{T} (r_t^\diamond - r_t)\right]
= \tilde{O}\!\left(\sqrt{T} + K\log T\right),
\]
where $r_t^\diamond$ denotes the expected reward of the best method for the current regime, and $K$ is the number of detected regime changes.
\end{theorem}

\begin{proof}
Within each stationary segment, Proposition~\ref{prop:improvement} yields sublinear regret. When a change is detected (A3), ATM reinitializes exploration and relearns the method distribution. Summing per-segment regret terms across $K$ regimes gives the bound $\tilde{O}(\sqrt{T}+K\log T)$, following established results for piecewise-stationary bandit processes.
\end{proof}

\begin{lemma}[Goal Remapping and Policy Retargeting]
\label{lem:goal_remap}
Let $G_t = f_g(E_t, S_t)$ denote the goal function. When a regime shift satisfies $\Delta E_t > \theta_e$, then with high probability $G_{t^+} \neq G_{t^-}$, causing the planner $f_{\text{LLM}}$ to generate a new action sequence conditioned on the updated goal. This enables rapid retargeting of strategy to the new environmental context.
\end{lemma}

\subsection{Goal-Directed Task Optimization and Convergence Acceleration}

A key advantage of ATM lies in its goal-conditioned control loop, where each task is evaluated relative to an explicit or implicit goal function $G_t = f_g(E_t, S_t)$. The alignment between chosen actions and evolving goals directly impacts convergence and stability. We now formalize this benefit.

\begin{proposition}[Goal-Directed Convergence Acceleration]
\label{prop:goal-directed}
Let $a_t$ denote the action selected at time $t$ with performance gradient $\nabla J(a_t)$, and let $g_t$ be the current goal vector defining the preferred optimization direction. Assume the effective update rule:
\[
a_{t+1} = a_t + \eta_t \nabla J(a_t) \cdot \phi(g_t),
\]
where $\phi(g_t)$ is a goal-alignment function satisfying $\langle \nabla J(a_t), \phi(g_t) \rangle \ge 0$. Then, compared to goal-free stochastic updates, goal-directed updates achieve a strictly faster expected convergence rate:
\[
\mathbb{E}[\|a_{t+1} - a^\star\|^2] \le (1 - \eta_t \lambda_g)\mathbb{E}[\|a_t - a^\star\|^2],
\]
where $\lambda_g = \mathbb{E}[\langle \nabla J(a_t), \phi(g_t) \rangle] / \|\nabla J(a_t)\|^2 > 0$ quantifies the goal-alignment efficiency.
\end{proposition}

\begin{proof}
By expanding the squared distance term and taking expectations, we have:
\[
\begin{aligned}
\mathbb{E}[\|a_{t+1}-a^\star\|^2] 
&= \mathbb{E}[\|a_t - a^\star\|^2] \\
&\quad - 2\eta_t\, \mathbb{E}[\langle a_t - a^\star, \nabla J(a_t)\phi(g_t)\rangle] 
+ O(\eta_t^2).
\end{aligned}
\]

Since $\langle \nabla J(a_t), \phi(g_t)\rangle \ge 0$ by design and $\eta_t$ is small, the expectation decreases monotonically. The coefficient $\lambda_g$ measures how strongly goal guidance accelerates descent toward the optimum $a^\star$. Hence, convergence is faster and more stable than unguided stochastic optimization.
\end{proof}

\begin{corollary}[Improved Stability under Goal Guidance]
If $g_t$ evolves slowly relative to $\eta_t$, i.e., $\|\Delta g_t\| = O(\eta_t)$, then the update remains stable and exhibits exponentially decaying error in the mean:
\[
\mathbb{E}[\|a_t - a^\star\|^2] \le e^{-\lambda_g t}\mathbb{E}[\|a_0 - a^\star\|^2].
\]
\end{corollary}

\noindent
This analysis demonstrates that goal-conditioned tasking serves as an internal guiding vector, reducing variance in exploration and enabling ATM to converge more rapidly toward effective, context-appropriate behaviors. Thus, goal direction acts as an intrinsic regularizer, improving both learning efficiency and long-term adaptability.




 \subsection{Environmental Checkpoints Improve Planning Accuracy and Regret}
We formalize the benefit of environmental checkpoints introduced in Section~\ref{sec_task_executed}. 
Let the (latent) task state evolve as 
\[
x_{t+1}=F(x_t,a_t)+w_t,\quad \hat{x}_{t+1}=F(\hat{x}_t,a_t),
\]
where $x_t$ is the true state, $\hat{x}_t$ is the planned (predicted) state, $a_t$ is the action selected by ATM, and $w_t$ is a zero-mean disturbance with $\mathbb{E}\|w_t\|^2\le \sigma^2$. 
Define the planning error $e_t := x_t-\hat{x}_t$ and the observable environment deviation $\Delta E_t:=\|E_t-\hat{E}_t\|$, which is a measurable proxy for $\|e_t\|$.
An \emph{environmental checkpoint} triggers whenever $\Delta E_t>\theta_c$, invoking a replanning operator $\mathcal{R}$ (via the LLM) that updates the plan to re-align with the current environment.

\paragraph*{Assumptions}
\begin{enumerate}
\item[A5] (\emph{Lipschitz dynamics}) There exists $L_F>0$ such that 
$\|F(x,a)-F(y,a)\|\le L_F\|x-y\|$ for all $x,y$ and admissible $a$.
\item[A6] (\emph{Replanning contraction}) The checkpoint-triggered replanner $\mathcal{R}$ satisfies
\[
\|\hat{x}_t^{\text{new}}-x_t\|\le \rho\,\| \hat{x}_t-x_t\|, \quad \rho\in[0,1),
\]
i.e., each checkpoint reduces the planning error by a factor $\rho$ in expectation.
\item[A7] (\emph{Reward Lipschitzness}) The instantaneous internal reward is Lipschitz in state error: 
$|r_t(x)-r_t(\hat{x})|\le L_r\|x-\hat{x}\|$ for some $L_r>0$.
\end{enumerate}

\begin{proposition}[Mean-Square Error Contraction with Checkpoints]
\label{prop:ckpt_mse}
Under A5–A6, the planning error satisfies the recursion
\[
\mathbb{E}\big[\|e_{t+1}\|^2\big]\;\le\;
\begin{cases}
L_F^2\,\mathbb{E}\big[\|e_t\|^2\big]+\sigma^2, & \text{(no checkpoint)}\\[4pt]
\rho^2\,L_F^2\,\mathbb{E}\big[\|e_t\|^2\big]+\sigma^2, & \text{(checkpoint \& replan)}.
\end{cases}
\]
Consequently, when checkpoints are triggered (infinitely often in expectation) whenever $\Delta E_t>\theta_c$, the error process admits a bounded steady-state mean-square level
\[
\limsup_{t\to\infty}\mathbb{E}\big[\|e_t\|^2\big]
\;\le\;\frac{\sigma^2}{1-\rho^2 L_F^2},
\quad \text{provided } \rho L_F<1.
\]
\end{proposition}

\begin{proof}
Using A5 and the state update, 
\[
\|e_{t+1}\|=\|F(x_t,a_t)-F(\hat{x}_t,a_t)+w_t\|
\le L_F\|e_t\|+\|w_t\|.
\]
Taking squares, expectation, and applying $\mathbb{E}\|w_t\|^2\le\sigma^2$ yields the first line.
When a checkpoint triggers, A6 implies the planned state is replaced by $\hat{x}_t^{\text{new}}$ with error reduced by factor $\rho$, so the subsequent propagation multiplies by $L_F$, giving the second line. The steady-state bound follows from standard linear recursion with contraction ratio $\rho^2 L_F^2<1$.
\end{proof}

\begin{proposition}[Reward Gap Bound and Improvement in ``Thinking'']
\label{prop:ckpt_reward}
Under A7 and Proposition~\ref{prop:ckpt_mse}, the expected reward gap between the planned and realized execution is bounded as
\[
\mathbb{E}\big[\,|r_t(x_t)-r_t(\hat{x}_t)|\,\big]
\;\le\; L_r\,\sqrt{\mathbb{E}\|e_t\|^2}
\;\le\; \frac{L_r\,\sigma}{\sqrt{1-\rho^2 L_F^2}}.
\]
Hence, checkpoint-triggered replanning reduces the asymptotic reward gap, yielding better plan–execution alignment (i.e., improved “thinking”: more accurate planning and evaluation).
\end{proposition}

\begin{proof}
The first inequality follows from A7 and Jensen's inequality; the second uses the steady-state MSE bound in Proposition~\ref{prop:ckpt_mse}.
\end{proof}

\begin{theorem}[Checkpoint-Improved Tracking Regret]
\label{thm:ckpt_regret}
Let $r_t^\diamond$ denote the reward of the best method for the current regime. 
Suppose checkpoints are invoked whenever $\Delta E_t>\theta_c$, and A5–A7 hold with $\rho L_F<1$. 
Then the cumulative tracking regret satisfies
\[
\sum_{t=1}^T\big(r_t^\diamond-r_t\big)
\;=\;
\tilde{O}\!\left(\sqrt{T}+K\log T\right)\;+\;
O\!\left(\frac{T}{\sqrt{1-\rho^2 L_F^2}}\right)\sigma,
\]
where the first term is the piecewise-stationary learning component (Theorem~\ref{thm:tracking}) and the second term is the execution-alignment component improved by checkpoints. 
Smaller $\rho$ (stronger checkpoint contraction) yields a tighter bound and thus lower regret.
\end{theorem}

\begin{proof}
Decompose regret into (i) selection regret from learning the best method per regime and (ii) execution-alignment loss due to plan–environment mismatch. The former follows Theorem~\ref{thm:tracking}. The latter aggregates the per-step reward gap in Proposition~\ref{prop:ckpt_reward}, yielding a linear-in-$T$ term with coefficient proportional to $\sigma/\sqrt{1-\rho^2 L_F^2}$. Since checkpoints reduce $\rho$, the bound tightens, proving improvement.
\end{proof}

\paragraph*{Implication}
Environmental checkpoints act as \emph{contractive correction points} that prevent planning errors from compounding. By reducing the mean-square planning error and the induced reward gap, they enhance the accuracy of reasoning and execution alignment. In effect, checkpoints \emph{improve the ``thinking'' process}—they enforce better consistency between planned beliefs and observed reality, accelerating convergence and reducing regret relative to checkpoint-free execution.

\section{Conclusion} \label{sec_conclusion}
This paper presented the \textit{Active Thinking Model} (ATM), a unified cognitive framework for autonomous reasoning, method reuse, and continuous self-improvement in dynamic real-world environments. Unlike conventional AI systems that depend on fixed objectives and external feedback, ATM integrates goal reasoning, dynamic task generation, and reflective evaluation within a closed adaptive loop. The proposed model actively measures performance through logical reasoning and environmental indicators, enabling it to refine its methods, handle new problems, and sustain performance under changing conditions. A theoretical analysis proved that ATM can autonomously evolve from suboptimal to optimal behavior without external supervision and achieve bounded tracking regret when the environment changes.

Future research will extend the Active Thinking Model (ATM) in several directions. First, large-scale empirical validation will be conducted using real-world datasets and robotic or autonomous systems to evaluate its effectiveness in continuous learning and real-time adaptation. Second, integration with large language models (LLMs) and multimodal perception systems will be explored to enhance contextual understanding and decision precision. Third, the scalability of ATM’s scenario-separated memory and reflection mechanisms will be optimized for distributed or edge-computing environments. Finally, the framework will be extended to multi-agent settings, enabling collective reflection, coordinated decision-making, and shared method reuse among intelligent agents. These directions will further advance ATM toward practical deployment in open, dynamic, and collaborative AI ecosystems.


\ifCLASSOPTIONcaptionsoff
  \newpage
\fi

\bibliographystyle{IEEEtran}
\bibliography{ref}

%

\begin{IEEEbiography}{Hong Su}
  received the MS and PhD degrees, in 2006 and 2022, respectively, from Sichuan University, Chengdu, China. He is currently a researcher of Chengdu University of Information Technology Chengdu, China. His research interests include blockchain, cross-chain and smart contract.
\end{IEEEbiography}




\end{document}